\pdfoutput=1

\documentclass[11pt,a4paper]{article}
\usepackage{url}

\usepackage[acceptedWithA]{tacl2021v1}

\usepackage{times}
\usepackage{latexsym}

\usepackage[T1]{fontenc}

\usepackage[utf8]{inputenc}

\usepackage{microtype}
\usepackage{amsfonts}
\usepackage{amssymb}
\usepackage{amsthm}
\usepackage{amsmath}
\usepackage{bbm}
\usepackage{enumitem}
\usepackage{linguex}
\usepackage{multirow}
\usepackage{makecell}
\usepackage[font=small]{caption}

\usepackage{pifont}
\newcommand{\cmark}{\ding{51}}%
\newcommand{\xmark}{\ding{55}}%

\usepackage{thmtools}

\newtheorem{prop}{Proposition}

\theoremstyle{definition}

\usepackage{caption}
\usepackage{subcaption}

\usepackage{tabularx,booktabs}
\usepackage{multicol}
\usepackage{bigdelim, rotating}

\newcommand{\citeposs}[1]{\citeauthor{#1}'s (\citeyear{#1})}

\DeclareMathOperator*{\E}{\mathbb{E}}

\newcommand{\bh}{\mathbf{h}}

\newcommand{\br}{\mathbf{r}}
\newcommand{\bR}{\mathbf{R}}

\newcommand{\bm}{\boldsymbol{m}}
\newcommand{\bM}{\boldsymbol{M}}
\newcommand{\bff}{\boldsymbol{f}}
\newcommand{\bl}{\boldsymbol{\ell}}
\newcommand{\bL}{\boldsymbol{L}}
\newcommand{\bz}{\boldsymbol{z}}
\newcommand{\bF}{\boldsymbol{F}}
\newcommand{\bZ}{\boldsymbol{Z}}

\newcommand{\ate}{\psi_{\text{ATE}}}
\newcommand{\ace}{\psi_{\text{ACE}}}
\newcommand{\baseline}{\psi_{\text{na\"ive}}}

\newcommand{\paired}{\psi_{\text{paired}}}

\newcommand{\calV}{\mathcal{V}}

\newcommand{\doo}{\mathrm{do}}
\newcommand{\masc}{\textsc{msc}}
\newcommand{\fem}{\textsc{fem}}
\newcommand{\sing}{\textsc{sing}}
\newcommand{\plr}{\textsc{plur}}

\newcommand{\rep}{\mathtt{tgt}}

\newcommand{\true}{$\mathtt{true}$}
\newcommand{\false}{$\mathtt{false}$}
\newcommand{\node}{$\mathtt{node}$}
\newcommand{\parent}{$\mathtt{parent}$}
\newcommand{\istate}{$\mathtt{state}$}
\newcommand{\isFocusNoun}{$\mathtt{isFocusNoun}$}
\newcommand{\isNSubj}{$\mathtt{nsubj}$}
\newcommand{\cls}{$\mathtt{[CLS]}$ }

\newcommand{\bhp}{\mathbf{h'}}
\newcommand{\bhatpaired}{\widehat{\bh}_{\paired}}
\newcommand{\bhatbaseline}{\widehat{\bh}_{\baseline}}

\usepackage{cleveref}
\crefname{section}{\S}{\S\S}
\Crefname{section}{\S}{\S\S}
\crefformat{section}{\S#2#1#3}
\crefname{figure}{Fig.}{Fig.}
\crefname{alg}{Alg.}{Alg.}
\crefname{line}{line}{lines}
\crefname{appendix}{App.}{}
\crefname{equation}{Eq.}{Eq.}
\crefname{table}{Table}{Tables}
\crefname{prop}{Proposition}{Propositions}

\usepackage{colortbl}
\usepackage{scalerel}

\usepackage{algorithm}
\usepackage[noend]{algpseudocode}
\algnewcommand{\parState}[1]{\State%
    \parbox[t]{\dimexpr\linewidth-\algmargin}{\strut\hangindent=\algorithmicindent \hangafter=1 #1\strut}}
\algrenewcommand\algorithmicindent{1.0em}%

\newcommand{\rightcomment}[1]{{\color{gray} \(\triangleright\) {\footnotesize\textit{#1}}}}
\algrenewcommand{\algorithmiccomment}[1]{\hfill \rightcomment{#1}}  
\algnewcommand{\LineComment}[1]{\State \rightcomment{#1}}
\algnewcommand{\LinesComment}[1]{\State \rightcomment{\parbox[t]{\linewidth-\leftmargin-\widthof{\(\triangleright\) }}{#1}}}

\newcommand{\algorithmicfunc}[1]{\textbf{def} #1 :}
\algdef{SE}[FUNC]{Func}{EndFunc}[1]{\algorithmicfunc{#1}}{}
\makeatletter
\ifthenelse{\equal{\ALG@noend}{t}}%
  {\algtext*{EndFunc}}
  {}%
\makeatother

\makeatletter
\algrenewcommand\ALG@beginalgorithmic{\small}
\makeatother

\definecolor{mutedblue}{HTML}{5DADE2}
\definecolor{mutedred}{HTML}{F1948A}
\definecolor{pastelgreen}{HTML}{40B0A6}
\definecolor{pastelyellow}{HTML}{E1BE6A}
\definecolor{silver}{HTML}{99A3A4}

\newcommand{\normal}{\colorbox{pastelgreen!30}{$\mathtt{NORMAL}$} }
\newcommand{\dir}{\colorbox{silver!30}{$\mathtt{DIR}$} }
\newcommand{\indir}{\colorbox{pastelyellow!30}{$\mathtt{INDIR}$} }

\newcommand{\term}[1]{\textbf{#1}}
\newcommand{\defn}[1]{\textbf{#1}}

\usepackage{todonotes}
\makeatletter
\newcommand*\iftodonotes{\if@todonotes@disabled\expandafter\@secondoftwo\else\expandafter\@firstoftwo\fi}  
\makeatother

\usepackage[normalem]{ulem}

\usepackage{inconsolata}
\title{Naturalistic Causal Probing for Morpho-Syntax}

\author{
Afra Amini$^{1,2}$~\;~Tiago Pimentel$^3$~\;~Clara Meister$^1$~\;~Ryan Cotterell$^{1,2}$ \\
  $^1$ETH Z\"{u}rich~\;~$^2$ETH AI Center~\;~$^3$University of Cambridge \\
 \texttt{\href{mailto:afra.amini@inf.ethz.ch
}{afra.amini@inf.ethz.ch
}}~\;~\texttt{\href{mailto:tp472@cam.ac.uk}{tp472@cam.ac.uk}}\\ \texttt{\href{mailto:clara.meister@inf.ethz.ch}{clara.meister@inf.ethz.ch}}~\;~\texttt{\href{mailto:ryan.cotterell@inf.ethz.ch}{ryan.cotterell@inf.ethz.ch}}
}

\begin{document}
\maketitle
\begin{abstract}
Probing has become a go-to methodology for interpreting and analyzing deep neural models in natural language processing.
However, there is still a lack of understanding of the limitations and weaknesses of various types of probes. 
In this work, we suggest a strategy for input-level intervention on naturalistic sentences.
Using our approach, we intervene on the morpho-syntactic features of a sentence, while keeping the rest of the sentence unchanged.
Such an intervention allows us to \emph{causally} probe pre-trained models.
We apply our naturalistic causal probing framework to analyze the effects of grammatical gender and number on contextualized representations extracted from three pre-trained models in Spanish: the multilingual versions of BERT, RoBERTa, and GPT-2.
Our experiments suggest that naturalistic interventions lead to stable estimates of the causal effects of various linguistic properties.
Moreover, our experiments demonstrate the importance of naturalistic causal probing when analyzing pre-trained models.

\vspace{1.5em}
\hspace{.5em}\includegraphics[width=1.25em,height=1.25em]{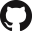}\hspace{.75em}\parbox{\dimexpr\linewidth-2\fboxsep-2\fboxrule}{\url{https://github.com/rycolab/naturalistic-causal-probing}}
\vspace{-.5em}
\end{abstract}

\section{Introduction}

Contextualized word representations are a byproduct of pre-trained neural language models and have led to improvements in performance on a myriad of downstream natural language processing (NLP) tasks \citep{joshi-etal-2019-bert, kondratyuk-2019-cross, zellers2019hellaswag,brown2020language}. 
Despite this performance improvement, though, it is still not obvious to researchers how these representations encode linguistic information.
One prominent line of work attempts to shed light on this topic through \term{probing} \citep{alain2016understanding}, also referred to as auxiliary prediction \cite{adi2016fine} or diagnostic classification \cite{hupkes2018visualisation}.
In machine learning parlance, a probe is a supervised classifier that is trained to predict a property of interest from the target model's representations. 
If the probe manages to predict the property with high accuracy, one may conclude that these representations encode information about the probed property.\looseness=-1

\begin{figure}
    \centering
    \includegraphics[width=\linewidth]{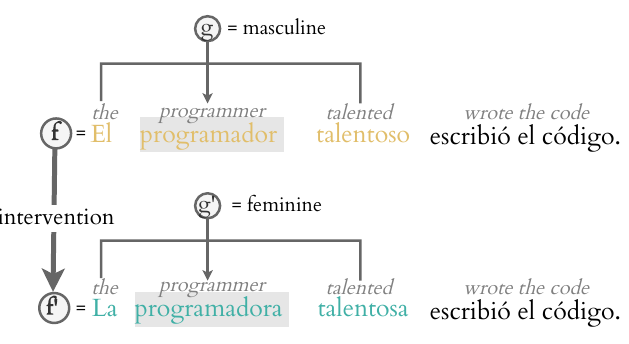}
    \caption{Intervention on the gender of lemma \textit{programador} (masculine $\rightarrow$ feminine). Changes are propagated from that noun to its dependent words accordingly.\looseness=-1}
    \vspace{-1em}
    \label{fig:causal-example}
\end{figure}

While widely used, probing is not without its limitations.\footnote{See \citet{belinkov2021probing} for an overview.}
For instance, probing a pre-trained model for grammatical gender can only tell us whether information about gender is \emph{present} in the representations,\footnote{See \citet{pimentel-etal-2020-information}, \citet{hewitt-etal-2021-conditional} and \citet{pimentel-cotterell-2021-bayesian} for fomalizations of this statement under information-theoretic frameworks.}
it cannot, however, tell us how or if the model actually uses information about gender in its predictions \citep{ravichander-etal-2021-probing, elazar2020bert, ravfogel-counter, lasri-etal-2022-probing}.
Furthermore, supervised probing cannot tell us whether the property under consideration is directly \emph{encoded} in the representations, or if it can be recovered from the representations alone due to spurious correlations among various linguistic properties.
In other words, while through supervised probing techniques we might find \emph{correlations} between a probed property and representations, we cannot uncover \emph{causal} relationships between them.\looseness=-1

In this work, we propose a new strategy for input-level intervention on naturalistic data to obtain what we call \defn{naturalistic counterfactuals}, which we then use to perform causal probing. 
Through such input-level interventions, we can ascertain whether a particular linguistic property has a \emph{causal} effect on a model's representations.
A number of prior papers have attempted to tease apart causal dependencies using either input-level or representation-level interventions.
For instance, work on \defn{representational counterfactuals} has investigated causal dependencies via interventions on neural representations.
While quite versatile, representation-level interventions make it hard---if not impossible---to determine whether we are only intervening on our property of interest.
Another proposed method, \defn{templated counterfactuals}, \emph{does} perform an input-level intervention strategy, ensuring that only the probed property will be affected.
Under such an approach, the researcher first creates a number of templated sentences (either manually or automatically), which they then fill with a set of minimal-pair words to generate counterfactual examples.
However, template-based interventions are limited by design: They do not reflect the diversity of sentences present in natural language, and, thus, lead to \emph{biased} estimates of the measured causal effects.
Naturalistic counterfactuals improve upon template-based interventions in that they lead to \emph{unbiased} estimates of the causal effect.\looseness=-1

In our first set of experiments, we employ naturalistic causal probing to estimate the average treatment effect (ATE) of two morpho-syntactic features---namely, number and grammatical gender---on a noun's contextualized representation. We show the estimated ATE's stability across corpora.
In our second set of experiments, we find that a noun's grammatical gender and its number are encoded by a small number of directions in three pre-trained models' representations: BERT, RoBERTa, and GPT-2.\footnote{
We study the Spanish version of these models, if it exists, or the multilingual version if there is no Spanish version.}
We further use naturalistic counterfactuals to causally investigate gender bias in RoBERTa. We find that RoBERTa is much more likely to predict the adjective \textit{hermoso(a)} (beautiful) for feminine nouns and \textit{racional} (rational) for masculine. 
This suggests RoBERTa is indeed gender biased in its adjective predictions. 

Finally, through our naturalistic counterfactuals, we show that correlational probes overestimate the presence of certain linguistic properties.
We compare the performance of correlational probes on two versions of our dataset: one unaltered and one augmented with naturalistic counterfactuals. 
While correlational probes achieve very high (above 90\%) performance when predicting gender from sentence-level representations, they only perform close to chance (around 60\%) on the augmented data.
Together, our results demonstrate the importance of a naturalistic causal approach to probing.\looseness=-1

\section{Probing} \label{sec:probes}

There are several types of probing methods that have been proposed for the analysis of NLP models, and there are many possible taxonomies of those methods.
For the purposes of this paper, we divide previously proposed probing models into two groups: correlational and causal probes.
On one hand, correlational probes attempt to uncover whether a probed property is \emph{present} in a model's representations. 
On the other hand, causal probes, roughly speaking, attempt to uncover how a model encodes and makes use of a specific probed property.
We compare and contrast correlational and causal probing techniques in this section.\looseness=-1

\subsection{Correlational Probing}
Correlational probing is any attempt to correlate the input representations with the probed property of interest.
Under correlational probing, the performance of a probe is viewed as the degree to which a model encodes information in its representations about some probed property \citep{alain2016understanding}. 
At various times, correlational results have been used to claim that language models have knowledge of various morphological, syntactic, and semantic phenomena \cite[][\emph{inter alia}]{adi2016fine,ettinger-etal-2016-probing,belinkov-etal-2017-evaluating,conneau-etal-2018-cram}. 
Yet the validity of these claims has been a subject of debate \citep{saphra-lopez-2019-understanding,hewitt-liang-2019-designing,pimentel-etal-2020-pareto,pimentel-etal-2020-information,voita-titov-2020-information}. 

\subsection{Causal Probing}
A more recent line of work aims to answer the question: What is the \emph{causal} relationship between the property of interest and the probed model's representations? 
In natural language, however, answering this question is not straightforward: sentences typically contain confounding factors that render analyses tedious.
To circumvent this problem, most work in causal probing relies on \textbf{interventions}, i.e., the act of setting a variable of interest to a fixed value \citep{pearl2009causal}.
Importantly, this must be done without altering any of this variable's causal parents, thereby keeping their probability distributions fixed.\footnote{Consider a set of three random variables with a causal structure $X \rightarrow Y \rightarrow Z$ (where $X$ causes $Y$, which causes Z). If we simply conditioned on $Y=1$, we would be left with the conditional distribution $p(x, z \mid Y=1) = p(x \mid Y=1)p(z \mid Y=1)$. If we perform an intervention on $Y=1$, on the other hand, we are left with a distribution of $p(x, z \mid \mathrm{do}(Y)=1) = p(x)p(z \mid Y=1)$; thus $X$'s distribution is not altered by $Y$.}
As a byproduct, these interventions generate \textbf{counterfactuals}, i.e., examples where a specific property of interest is changed while everything else is held constant.
Counterfactuals can then be used to perform a causal analysis.
Prior probing papers have proposed methods using both representational and templated counterfactuals, as we discuss next.\looseness=-1

\paragraph{Representational Counterfactuals.}
A few recent causal probing papers perform interventions directly on a model's representations \citep{giulianelli2018under,feder2020causalm,vig2020investigating,tucker2021if,ravfogel-counter,lasri-etal-2022-probing, ravfogel+al.icml22}.
For example, \citet{elazar2020bert} use iterative null space projection \citep[INLP;][]{ravfogel-etal-2020-null} to remove an analyzed property's information, e.g., part of speech, from the representations.
Although representational interventions can be applied to situations where other forms of intervention are not feasible,
it is often impossible to make sure only the information about the probed property is removed or changed.\footnote{There are, however, methods to mitigate this issue, e.g.,
\citet{ravfogel2022adversarial} recently proposed an improved (adversarial) method to remove information from a set of representations which greatly reduces the number of removed dimensions.}
In the absence of this guarantee, any causal conclusion should be viewed with caution.\looseness=-1

\paragraph{Templated Counterfactuals.}
Other works \citep{vig2020investigating,finlayson-etal-2021-causal}, like us, have leveraged input-level interventions.  
However, in these cases, the interventions are carried out using templated minimal-pair sentences, which differ only with respect to a single analyzed property. 
Using these minimal pairs, they estimate the effect of an input-level intervention on individual attention heads and neurons. 
One benefit of template-based approaches is that they create a highly controlled environment, which guarantees that the intervention is done correctly, and which may lead to insights that would be impossible to gain from natural data. 
However, since the templates are typically designed to analyze a specific property, they cover a narrow set linguistic phenomena, which may not reflect the complexity of language in naturalistic data.\looseness=-1

\paragraph{Naturalistic Counterfactuals.}
In this paper, following \citet{zmigrod-etal-2019-counterfactual}, we propose a new and less complex strategy to perform input-level interventions by creating naturalistic counterfactuals that are \emph{not} derived from templates.
Instead, we derive the counterfactuals from the dependency structure of the sentence.
By creating counterfactuals on the fly using a dependency parse, we avoid the biases of manually creating templates.
Furthermore, our approach guarantees that we only intervene on the specific linguistic property of interest, e.g., changing the grammatical gender or number of a noun.
\looseness=-1

\section{The Causal Framework} 

The question of interest in this paper is how contextualized representations are \emph{causally} affected by a morpho-syntactic feature such as gender or number.
To see how our method works, it is easiest to start with an example.
Let's consider the following pair of Spanish sentences:

{\small
\ex. \label{ex:msc} \textit{\textbf{El}} \textit{programador} \textit{talentos\textbf{o}} \textit{escribi{\'o}} \textit{el} \textit{c{\'o}digo}. \\ 
the.\textsc{m}.\textsc{sg}
programmer.\textsc{m}.\textsc{sg}
talented.\textsc{m}.\textsc{sg}
wrote the code. \\
The talented programmer wrote the code.

\ex. \label{ex:fem} \textit{\textbf{La}} \textit{programador\textbf{a}} \textit{talentos\textbf{a}} \textit{escribi{\'o}} \textit{el} \textit{c{\'o}digo}. \\ 
the.\textsc{f}.\textsc{sg}
programmer.\textsc{f}.\textsc{sg}
talented.\textsc{f}.\textsc{sg}
wrote the code. \\
The talented programmer wrote the code.

}

The meaning of these sentences is equivalent up to the gender of the noun \textit{programador}, whose feminine form is \textit{programador\textbf{a}}. However, more
than just this one word changes from \ref{ex:msc} to \ref{ex:fem}: The definite article \textbf{\textit{el}} changes to \textbf{\textit{la}} and the adjective \textit{talentos\textbf{o}} changes to \textit{talentos\textbf{a}}. 
In the terminology of this paper, we will refer to \textit{programador} as the \term{focus noun}, as it is the noun whose grammatical properties we are going to change.
We will refer to the changing of \ref{ex:msc} to \ref{ex:fem} as a \term{syntactic intervention} on the focus noun.
Informally, a syntactic intervention may be thought of as taking part in two steps. 
First, we swap the focus noun (\textit{programador}) with another noun that is equivalent up to a single grammatical property.
In this case, we swap \textit{programador} with \textit{programador\textbf{a}} which differs only in its gender marking.
Second, we reinflect the sentence so that all necessary words grammatically agree with the new focus noun.
The result of a syntactic intervention is a pair of sentences that differ minimally, i.e., only with respect to this one grammatical property. 
Another way of framing the syntactic intervention is as a counterfactual: What would \ref{ex:msc} have looked like if \textit{programador} had been feminine?
The rest of this section focuses on formalizing the notion of a syntactic intervention and discussing how to use them in a causal inference framework for probing.

\paragraph{A Note on Inanimate Nouns.} 
When estimating the effect of grammatical gender here, we restrict our investigation to \emph{animate} nouns, e.g., \textit{programador\textbf{a}}/\textit{programador} (feminine/masculine programmer).
Grammatical gender of inanimate nouns is lexicalized, meaning that each noun is assigned a single gender, e.g., \textit{puente} (bridge) is masculine. In other words, there is not a non-zero probability of assigning each lemmata to each gender, which violates a condition called \textbf{positivity} in causal inference literature. 
Thus, we cannot perform an intervention on the grammatical gender of those words, but rather would need to perform an intervention on the lemma itself. 
We refer to \citet{gonen-etal-2019-grammatical} for an analysis of the effect of gender on inanimate nouns' representations.
Note that a similar lexicalization can also be observed in a few animate nouns, e.g. \textit{madre/padre} (mother/father). 
In such cases, to separate the lemma from gender, we assume that these words share a hypothetical lemma, which in our example represents parenthood, and combining that with gender would give us the specific forms, e.g. \textit{madre}/\textit{padre}.\looseness=-1

\begin{figure*}[t]
    \centering
    \begin{subfigure}{0.49\textwidth}
    \includegraphics[width=\linewidth]{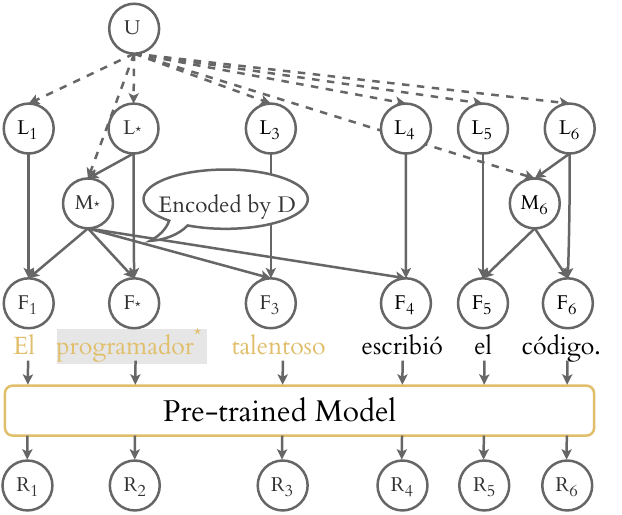}
    \end{subfigure}
    \begin{subfigure}{0.49\textwidth}
    \includegraphics[width=\linewidth]{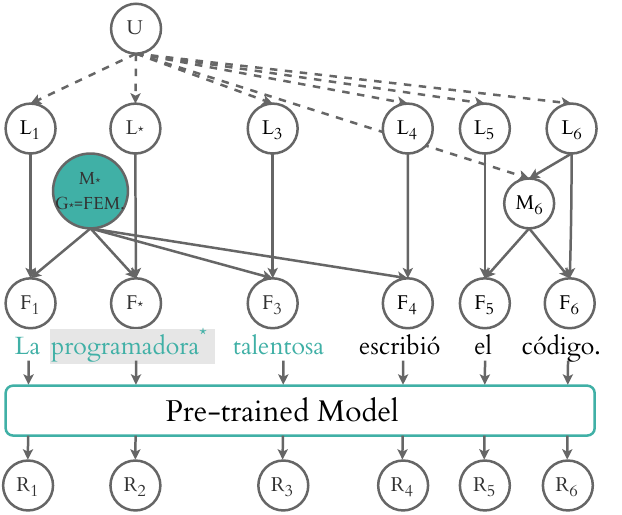}
    \end{subfigure}
    \caption{Causal graph for the Spanish sentence \textit{\textbf{El} programador talentos\textbf{o} escribió el código.} before (on the left) and after (on the right) an intervention on the grammatical gender of the focus noun.}
    \label{fig:detailed-causal}
    \vspace{-10pt}
\end{figure*}

\subsection{The Causal Model} \label{sec:cm}

We now describe a causal model that will allow us to more formally discuss syntactic interventions.

\paragraph{Notation and Variables.}
We denote random variables in upper-case letters and instances with lower-case letters. 
We bold sequences: bold lower-case letters represent a sequence of words and bold upper-case letters represent a sequence of random variables.
Let $\bff = \langle f_1, \ldots, f_T\rangle$ be a sentence (of length $T$) where each $f_t$ is a word form. In addition, let $\br$ be the list of contextual representations $\br = \langle r_1, \ldots, r_T \rangle$ where each $r_t \in \mathbb{R}^h$, and is in one-to-one correspondence with the sentence $\bff$, i.e., $r_t$ is $f_t$'s contextual representation. 
Furthermore, let  $\bl = \langle \ell_1, \ldots, \ell_T\rangle$ be a list of lemmata and $\widetilde{\bm} = \langle m_1, \ldots, m_T \rangle$ a list of morpho-syntactic features co-indexed with $\bff$; $\ell_t$ is the lemma of $f_t$ and $m_t$ is its morpho-syntactic features. 
We call $\bm = \langle m_{t_1}, \ldots, m_{t_K} \rangle$ the \textbf{minimal list of morpho-syntactic features}, where each $t_k$ is an index between 1 to $T$. 
In essence, we drop features of the tokens that are dependent on other tokens' morphology. 
In our example \ref{ex:msc} this means we only include the morpho-syntactic features of \textit{programador} and \textit{c{\'o}digo}, thus $\bm = \langle m_2, m_6 \rangle$.\footnote{In this work, we only focus on two morpho-syntactic features: gender and number. To analyze other features, the minimal list should be expanded, e.g. to analyze verb tense, $m_3$ should be added to the list.} 
We denote the morpho-syntactic feature of interest as $m_*$, which, in this work, represents either the gender $g_*$ or number $n_*$ of the focus noun.
We further denote the lemma of the focus noun as $\ell_*$. \looseness=-1

\paragraph{Causal Assumptions.} 
Our causal model is introduced in \cref{fig:detailed-causal}.
It encodes the causal relationships between $U, \bL, \bM, \bF$ and $\bR$. Explicitly, we assume the following causal relationships:

\begin{itemize}[leftmargin=*]
    \item \emph{$\bM$ and $\bL$ are causally dependent on $U$}. The underlying meaning that the writer of a sentence wants to convey determines the used lemmas and morpho-syntactic features;
    \item In general, \emph{$L_t$ can causally affect $M_t$.} Take the gender of inanimate nouns as an example, where the lemma determines the gender;
    \item \emph{$\bF$ is causally dependent on $\bL$ and $\bM$}. Word forms are a combination of lemmata and morpho-syntactic features;
    \item \emph{$\bR$ is causally dependent on $\bF$}. Contextualized representations are obtained by processing the sentences through the probed model.
\end{itemize}

\paragraph{Dependency Trees.} 
In order to measure the causal effect of the gender of the focus noun ($g_*$) on the contextualized representation ($\br$), all of its causal dependencies must be considered. 
As our causal graph shows (in \cref{fig:detailed-causal}), $g_*$ not only has a causal effect on the focus noun's form, but also on the definite article \textbf{\textit{el}} and the adjective \textit{talentos\textbf{o}}. Yet, not all word forms in a sentence are affected; for instance, the definite article \textbf{\textit{el}} in the noun phrase \textit{el c{\'o}digo}.
Luckily, within a given sentence, such relationships are naturally encoded by that sentence's dependency tree. The dependency graph $d$ of a sentence $\bff$ is a directed graph created by connecting each word form $f_t$ for $1 \leq t \leq T$ to its syntactic parent. 
We use the information encoded in $d$ by leveraging the fact that a word form $f_t$ is causally dependent on its syntactic parent. 
In essence, a dependency tree $d$ implicitly encodes a function $d_t[\bm]$ which returns the subset of morphological properties that causally affect the form $f_t$. 
Thus, we are able to express the complete joint probability distribution of our causal model as follows:\looseness=-1
\begin{align}
    p(\bff, \bm&,\bl, u) \\
    &= p(u)\,p(\bm, \bl \mid u)\,p(\bff \mid \bm, \bl) \nonumber \\
    &= p(u)\,p(\bm, \bl \mid u)\,\prod_{t=1}^{T} p(f_t \mid d_t[\bm], \ell_t)  \nonumber 
\end{align}

\paragraph{Abstract Causal Model.} 
We can now simplify the causal model from \cref{fig:detailed-causal}
into \cref{fig:causal-abstract}. For simplicity, we isolate the lemma and morpho-syntactic feature of interest $L_*$ and $M_*$ and aggregate the other lemmata and morpho-syntactic features into an abstract variable, which we call $\bZ$ and refer to as the \textbf{context}. Furthermore, we only show the aggregation of word forms and representations as $\bF$ and $\bR$ in the abstract model. 
We will assume for now, and in most of our experiments, that the output of the causal model ($\bR$ in \cref{fig:causal-abstract}) represents the contextualized representation of the focus noun.
However, as we generalize later, the output of the causal model can be any function of word forms $\bF$, such as: the representation of other words in the sentence, the probability distribution assigned by the model to a masked word, or even the output of a downstream task.
We note that \cref{fig:causal-abstract} can be easily re-expanded into \cref{fig:detailed-causal} for any specific utterance by using its dependency tree. \looseness=-1

\begin{figure}
    \centering
    \includegraphics[width=\linewidth]{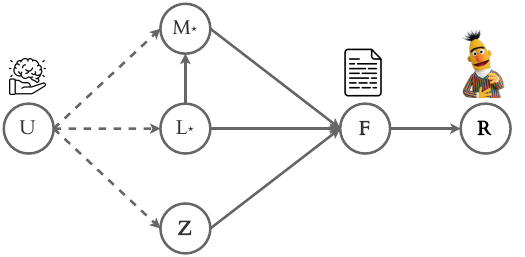}
    \caption{Causal model showing dependencies between the underlying meaning ($U$), lemma ($L_*$) and morpho-syntactic features ($M_*$) of the focus noun, context ($\bZ$), sentences ($\bF$) and contextualized representations ($\bR$).}
    \label{fig:causal-abstract}
    \vspace{-10pt}
\end{figure}
\subsection{Naturalistic Counterfactuals} \label{sec:intervention} \label{sec:algorithm}
In causal inference literature, the $\doo(\cdot)$ operator represents an intervention on a causal diagram. 
For instance, we might want to intervene on the gender of the focus noun (thus using gender $G_*$ as the morpho-syntactic feature of interest $M_*$). 
Concretely, in our example (\cref{fig:detailed-causal}), $\doo(G_{*} = \fem)$ means intervening on the causal graph by removing all the causal edges going into $G_*$ from $U$ and $L_*$ and setting $G_{*}$'s value to a specific realization \fem. 
The result of this intervention on a sampled sentence $\bff$ is a new counterfactual sentence $\bff'$. 
As our causal graph suggests, the relationship between words in a sentence is complex, occurring at multiple levels of abstraction; swapping the gender of a single word---while leaving all other words unchanged---may not result in grammatical text.  
Consequently, one must approach the creation of counterfactuals in natural language with caution.
Specifically, we rely on syntactic interventions to generate our naturalistic counterfactuals.

\paragraph{Syntactic Intervention.}
We develop a heuristic algorithm to perform our interventions, shown in \cref{app:alg}.
Given a sentence and its dependency tree, the algorithm generates a counterfactual version of the sentence, i.e., approximating the $\doo(\cdot)$ operation.
This algorithm processes the dependency tree of each sentence in a depth-first search recursive manner. In each iteration, if the node in process is a noun, it is marked as the focus noun\footnote{Specifically, for gender intervention we only mark the noun as the focus if it is an animate noun.} and a new copy of the sentence is created, which will be the base of the counterfactual sentence.
Then, the intervention is performed, altering the focus noun and all dependent tokens in the copied sentence.\footnote{This is a simplified version of the algorithm where we omit the rule-based re-inflection functions for nouns, adjectives and determiners. We also handle contractions, such as \textit{a} $+$ \textit{el} $\rightarrow$ \textit{al}, which is not mentioned in this pseudo-code.}
Notedly, when we syntactically intervene on the \emph{grammatical gender} or \emph{number} of a noun, we do not alter potentially incompatible semantic contexts. 
Take sentence \ref{ex:err} as an example, where the focus noun is \textit{mujer} and we intervene on \emph{gender}. Its counterfactual sentence \ref{ex:err_masc} is semantically odd and unlikely, but still meaningful.
We can thus estimate the causal effect of grammatical gender in the contextual representations---breaking the correlation between morpho-syntax and semantics.\looseness=-1

\ex. \label{ex:err} \textit{\textbf{La}} \textit{\textbf{mujer}} \textit{dio a luz a 6 bebés}. \\ 
the.\textsc{f}.\textsc{sg}
woman.\textsc{f}.\textsc{sg}
gave birth to 6 babies. \\
The woman gave birth to 6 babies.

\ex. \label{ex:err_masc} \textit{\textbf{El}} \textit{\textbf{hombre}} \textit{dio a luz a 6 bebés}. \\ 
the.\textsc{m}.\textsc{sg}
man.\textsc{m}.\textsc{sg}
gave birth to 6 babies. \\
The man gave birth to 6 babies.

\subsection{Measuring Causal Effects}

In this section, we define the causal effect of a morpho-syntactic feature.
We then present estimators for these values in the following section.
While we focus on grammatical gender here, our derivations are similarly applicable to number and other morpho-syntactic features.

Given a specific focus--context pair $(\ell_*, \bz)$,
the causal effect of gender $G_*$ on the representations is called the \term{individual treatment effect} \cite[ITE;][]{rosenbaum1983central} and is defined as:\looseness=-1%
\begin{align} \label{eq:ite} 
     &\Delta(\ell_*, \bz) = \\
     &\quad \E_{\bF} \Big[ \rep(\bF) \mid G_*\!=\!\masc, L_*\!=\!\ell_*, \bZ\!=\!\bz \Big] \nonumber\\
     &\quad - \E_{\bF} \Big[ \rep(\bF) \mid G_*\!=\!\fem, L_*\!=\!\ell_*, \bZ\!=\!\bz \Big] \nonumber
\end{align}
where $\rep(\cdot)$ is a deterministic function that implements the model being probed, e.g., a pretrained model like BERT, taking a form $\bF$ as input and outputting $\bR$.
Since $\bF$ is itself a deterministic function of a $\langle G_*, L_*, \bZ \rangle$ triple, we can rewrite this equation as:\footnote{We overload $\rep(\cdot)$ to receive either $\bF$ or $\langle G_*, L_*, \bZ \rangle$.}
\begin{align} \label{eq:ite-simple} 
     &\Delta(\ell_*, \bz) =  \\
     &\quad \rep(\masc, \ell_*, \bz) - \rep(\fem, \ell_*, \bz) \nonumber
\end{align}
As can be seen from \cref{eq:ite-simple}, the ITE is the difference in the representation given that the focus noun of the sentence is masculine vs. feminine.

To get a more general understanding of how gender affects these representations, however, it is not enough to just look at individual treatment effects.
It is necessary to consider a holistic metric across the entire language.
The \term{average treatment effect} (ATE) is one such metric, and is defined as the difference between the following expectations:
\begin{align} \label{eq:ate-base}
    \ate  = &\E_{\bF} \big[\rep(\bF) \mid \doo(G_*=\masc) \big] \\
    &\quad- \E_{\bF} \big[\rep(\bF) \mid \doo(G_*=\fem) \big] \nonumber
\end{align}
In words, the ATE is the expected causal effect of one random variable on another (in this case gender on the model's representations).
Computing this expectation, however, is not as simple as conditioning it on gender.
As there are backdoor paths\footnote{A backdoor path is a causal path from an analyzed variable to its effect which contains an arrow \emph{to} the treatment (i.e., an arrow going backwards). For instance, consider random variables with a causal structure $Y \rightarrow X \rightarrow Z$ and $Y \rightarrow Z$ (where Y causes X, and both X and Y cause Z). $X \leftarrow Y \rightarrow Z$ forms a backdoor path
\citep[Definition 3,][]{pearl2009causal}.
}  
from the treatment (gender) to the effect (the representations),
we rely on the \term{backdoor criterion} \citep{pearl2009causal} to compute this expectation. Simply put, we first need to find a set of variables that block every such backdoor path.
We then condition our expectation on them. As shown in \cref{prop1} (in the appendix), the set of variables satisfying the backdoor criterion in our case is $\{L_*, \bZ \}$. Therefore, we can rewrite \cref{eq:ate-base} by conditioning our expectation over$\{L_*, \bZ \}$:
\begin{align} \label{eq:ate} 
     &\ate = \\ &\quad\E_{L_*, \bZ} \left[\E_{\bF} \Big[ \rep(\bF) \mid G_* = \masc, L_*, \bZ \Big]\right] \nonumber \\
     & \quad -\E_{L_*, \bZ} \left[\E_{\bF} \Big[ \rep(\bF) \mid G_* = \fem, L_*, \bZ \Big] \right] \nonumber
\end{align}
which we can again rewrite as:%
\begin{align} \label{eq:ate-simple} 
     &\ate = \\
     &\quad \E_{L_*, \bZ} \left[\rep(\masc, L_*, \bZ) - \rep(\fem, L_*, \bZ ) \right] \nonumber
\end{align}
Furthermore, plugging \cref{eq:ite-simple} into \cref{eq:ate-simple}:
\begin{equation} \label{eq:ate_as_delta}
    \ate = \E_{L_*, \bZ} \big[ \Delta(L_*, \bZ) \big]
\end{equation}
reveals that \cref{eq:ate} is just the ITE in expectation.
Thus, the ATE is an appropriate language-wide measure of the effect of gender on contextual representations.\looseness=-1

\section{Approximating the ATE}\label{sec:approx}

In this section, we show how to estimate \cref{eq:ate-simple} from a finite corpus of sentences $\mathcal{S}$. 

\subsection{Na\"ive Estimator}

Each sentence in our corpus can be written as a triple $\langle g_*, \ell_*, \bz \rangle$.
We now discuss how to use such a corpus to estimate \cref{eq:ate-simple}.
Specifically, we first compute the sample mean using two subsets of sentences: one with only masculine focus nouns $\mathcal{S}_\masc$ and the other with feminine ones $\mathcal{S}_\fem$. We then compute their difference:%
\begin{align} \label{eq:ate-baseline} 
    &\baseline = \\
      &\quad\, \frac{1}{|\mathcal{S}_{\masc}|} \sum_{\langle \_, \ell_*, \bz \rangle \in \mathcal{S}_\masc} \rep(\masc, \ell_*, \bz) \nonumber \\
     &\qquad\,\, - \frac{1}{|\mathcal{S}_{\fem}|} \sum_{\langle \_, \ell_*, \bz \rangle \in \mathcal{S}_\fem} \rep(\fem, \ell_*, \bz) \nonumber
\end{align}
We note, however, that this is a very na\"ive estimator.\footnote{This is referred to as the na\"ive or unadjusted estimator in the literature \citep{hernan2020causal}.}
Since $\mathcal{S}_\masc$ (and respectively $\mathcal{S}_\fem$) includes only the fraction of sentences with masculine focus nouns, restricting the sample mean to this set of instances is equivalent to using samples $\bz, \ell_* \sim p(\bz, \ell_* \mid \masc)$, rather than $\bz, \ell_* \sim p(\bz, \ell_*)$ (as should be done for ATE). 
Notably, this is equivalent to ignoring the $\doo$ operator in \cref{eq:ate-base}.
Consequently, \cref{eq:ate-baseline} introduces a purely correlational baseline.
In the following section, we present our (better) causal estimator.\looseness=-1

\subsection{Paired Estimator}
We now use our naturalistic counterfactual sentences to approximate the ATE. Specifically, by relying on our syntactic interventions, we can get both a feminine and masculine form of each sentence $(\ell_*, \bz)$ sampled from the corpus.
Concretely, we use the following \textbf{paired} estimator:%
\begin{align} \label{eq:ate-paired} 
     &\paired = \\ 
     &\frac{1}{|\mathcal{S}|} \sum_{\langle \_, \ell_*, \bz \rangle \in \mathcal{S}} \Big[\underbrace{
     \rep(\masc, \ell_*, \bz)
     }_{(1)} -
     \underbrace{
    \rep(\fem, \ell_*, \bz)
    }_{(2)} \Big] \nonumber
\end{align}
where, depending on $g_*$, the model's output $\rep(\cdot)$ in (1) and (2) will be extracted from a pre-trained model using either the original or counterfactual sentences.\looseness=-1 

\subsection{A Closer Look at our Estimators}

A closer look at our paired estimator in \cref{eq:ate-paired} shows that it is an \emph{unbiased} Monte Carlo estimator of the ATE presented in \cref{eq:ate-simple}.
In short, if we assume our corpus $\mathcal{S}$ was sampled from the target distribution, we can use this corpus as samples $\ell_*, \bz \sim p(\ell_*, \bz)$.
For each $\ell_*, \bz$ pair, we can then generate sentences with both $\masc$ and $\fem$ grammatical genders to estimate the ATE.

The na\"ive estimator, on the other hand, will not produce an unbiased estimate of the ATE.
As mentioned above, by considering sentences in $\mathcal{S}_\masc$ or $\mathcal{S}_\fem$ separately, we implicitly condition on the gender when approximating each expectation.
This estimator instead approximates a value we term the \defn{average correlational effect} (ACE):
\begin{align} \label{eq:ace} 
     &\ace = \E_{L_*, \bZ \mid G_*=\masc} \left[\rep(\masc, L_*, \bZ)\right] \\
     &\qquad\qquad -\E_{L_*, \bZ \mid G_*=\fem} \left[\rep(\fem, L_*, \bZ) \right] \nonumber
\end{align}

On a separate note, template-based approaches allow the researcher to investigate causal effects by using minimal pairs of sentences, each
of which can be used to estimate an ITE (as in \cref{eq:ite-simple}). 
And, by averaging them, they provide an estimate of ATE (as in \cref{eq:ate_as_delta}). 
However, these minimal pairs are either manually written or automatically collected using template structures. 
Therefore, they cover a narrow (and potentially biased) set of structures, arguably not following a naturalistic distribution.
In other words, their corpus $\mathcal{S}$ cannot be assumed to be sampled according to the distribution $p(\ell_*, \bz)$.\footnote{This becomes clear when we take a look at the sentences in one of such template-based datasets. For instance, all sentences in the Winogender dataset \citep{winogender}---used by \citet{vig2020investigating}---have very similar sentential structures.
Such biases, however, are not necessarily problematic and might be imposed by design to analyze specific phenomena.\looseness=-1}
In practice, templated counterfactuals approximate the treatment effect using an approach identical to the paired estimators---up to a change of distribution. This change of distribution, however, may lead to biased estimates of the ATE.\looseness=-1 

\section{Dataset}
\begin{table}
\centering
\resizebox{\columnwidth}{!}{%
\begin{tabular}{lccccccc}\toprule
&&&& \multicolumn{2}{c}{Gender} & \multicolumn{2}{c}{Number}\\
\cmidrule(lr){5-6} \cmidrule(lr){7-8} 
& train & dev & test & \masc & \fem & \sing & \plr \\
\midrule
\multirow{2}{*}{AnCora} & \cmark & \cmark & \xmark & $1{,}029$ & 203 & $14{,}602$ & $6{,}692$ \\
& \xmark & \xmark & \cmark & 107 & 21 & $1{,}540$ & 693 \\
\midrule
GSD & \cmark & \cmark & \xmark & 403 & 135 & $9{,}141$ & $3{,}993$ \\
\bottomrule
\end{tabular}
}
\caption{Aggregated dataset statistics}
\label{tab:dataset-stats}
\end{table}

We use two Spanish UD treebanks \citep{nivre-etal-2020-universal} in our experiments: Spanish-GSD \citep{mcdonald-etal-2013-universal} and Spanish-AnCora \citep{taule-etal-2008-ancora}.
We only analyze gender on animate nouns and use Open Multilingual WordNet \citep{Gonzalez-Agirre:Laparra:Rigau:2012} to mark the animacy.
Corpus statistics for the datasets can be found in \Cref{tab:dataset-stats}.\looseness=-1

\subsection{Evaluating Counterfactual Sentences} 

To evaluate our syntactic intervention algorithm (introduced in \cref{sec:algorithm}), we randomly sample a subset of 100 sentences from our datasets.
These samples are evenly distributed across the two datasets (AnCora and GSD), morpho-syntactic features (gender and number), and categories within each feature (masculine, feminine, singular, and plural). 
A native Spanish speaker assessed the grammaticality of sampled sentences. 
Our syntactic intervention algorithm was able to accurately generate counterfactuals for 73\% of the sentences.\footnote{Approximating our estimate of this accuracy with a normal distribution, we obtain a 95\% confidence interval (Wald interval) which ranges from 64\% to 82\% \citep{interval-brown}.}
The accuracy for the gender and number interventions are 76\% and 70\%, respectively.
Due to the subtleties discussed in disentangling syntax from semantics and the complex sentence structures found in naturalistic data, we believe this error is within an acceptable range and leave improvements to future work.\looseness=-1

\subsection{Template-Based Dataset}\label{sec:templated-data}

To compare our approach to templated counterfactuals, we translate two datasets for measuring gender bias: Winogender \citep{winogender} and WinoBias \citep{winobias}.
As shown by \citet{stanovsky-etal-2019-evaluating}, simply translating these templates to Spanish leads to biased translations, where professions are translated stereotypically and the context is ignored.
Following \citeauthor{stanovsky-etal-2019-evaluating}, we thus put either \textit{handsome} and \textit{pretty} before nouns to enforce the gender constraint after translation.
Consider, for instance, the sentence: ``\textit{\textbf{The developer} was unable to communicate with the writer because \textbf{he} only understands the code.}'' 
We rewrite it as ``\textit{The handsome developer$\ldots$}''.
Similarly, if the pronoun was \textit{she}, we would write ``\textit{The pretty developer$\ldots$}''. 
As an extra constraint, we want to ensure the gender of the \textit{writer} stays the same before and after the intervention.
Therefore, we make two copies of the sentence: One where \textit{writer} is translated as \textit{escritor\textbf{a}} (feminine writer), enforced by replacing \textit{writer} with \textit{pretty writer}, and one where \textit{writer} is translated as \textit{escritor} (masculine writer), enforced by replacing \textit{writer} with \textit{handsome writer}.
We translate the resulting pairs of sentences using the Google Translate API and drop the sentences with wrong gender translations.
In the end, we obtain 2740 minimal pairs.

\section{Insights From ATE Estimators}

In the following experiments, we first use the estimators introduced in \cref{sec:approx} to approximate the ATE of number and grammatical gender on contextualized representations. 
We look at how stable these ATE estimates are across datasets, and whether they change across words with different parts of speech.
We then analyze whether the ATE (as an expected value) was an accurate description of how representations actually change in individual sentences.
Finally, we compute the ATE of gender on
the probability of predicting specific adjectives in a sentence, thereby measuring the causal effect of gender in adjective prediction.\looseness=-1

\begin{figure}
\centering
    \includegraphics[width=\linewidth]{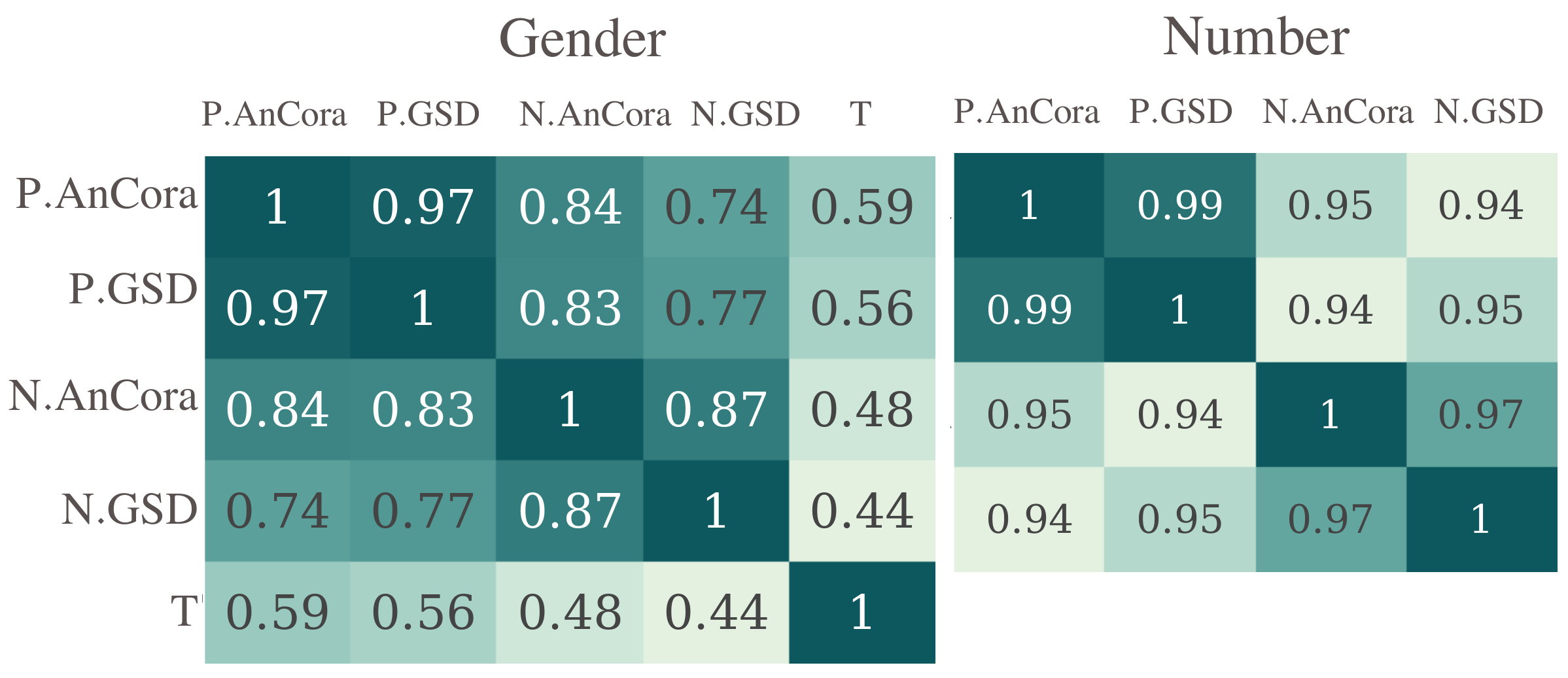}
    \caption{Cosine similarities of the ATE on BERT representations.
    N. represents $\baseline$; P. represents $\paired$; and T. represents $\paired$ estimated on the template-based dataset.
    }
    \label{fig:cos-ate}
    \vspace{-7pt}
\end{figure}

\subsection{Variations across ATEs}

\paragraph{Variation Across Datasets.} 
Using our ATE estimators, we compute the average treatment effect of both gender and number on BERT's contextualized representations \citep{devlin-etal-2019-bert} of focus nouns.\footnote{More specifically \textsc{bert-base-multilingual-cased} in the Transformers library \citep{wolf-etal-2020-transformers}.}
We compute the $\paired$ and $\baseline$ estimators. 
\cref{fig:cos-ate} presents their cosine similarities.
We observe high cosine similarities between paired estimators across datasets,\footnote{To make sure that the imbalance in the dataset \emph{before} intervention doesn't have a significant effect on results, we create a balanced version of the dataset, where we observe similar results.} but lower cosine similarities with the na\"ive estimator. 
This suggests that, while the causal effect is stable across treebanks, the correlational effect is more susceptible to variations in the datasets, e.g., semantic variations due to the domain from which treebanks were sampled. 

\paragraph{Templated vs.\@ Naturalistic Counterfactuals.} 
As an extra baseline, we estimate the ATE using a paired estimator with the template-based dataset introduced in \cref{sec:templated-data}
We observe a low cosine similarity between our naturalistic ATE estimates and the template-based ones. 
This shows that sentences from template-based datasets are substantially different from naturalistic datasets, thus failing to provide unbiased estimates in naturalistic settings.\looseness=-1

\begin{figure}
    \centering
    \includegraphics[width=\linewidth]{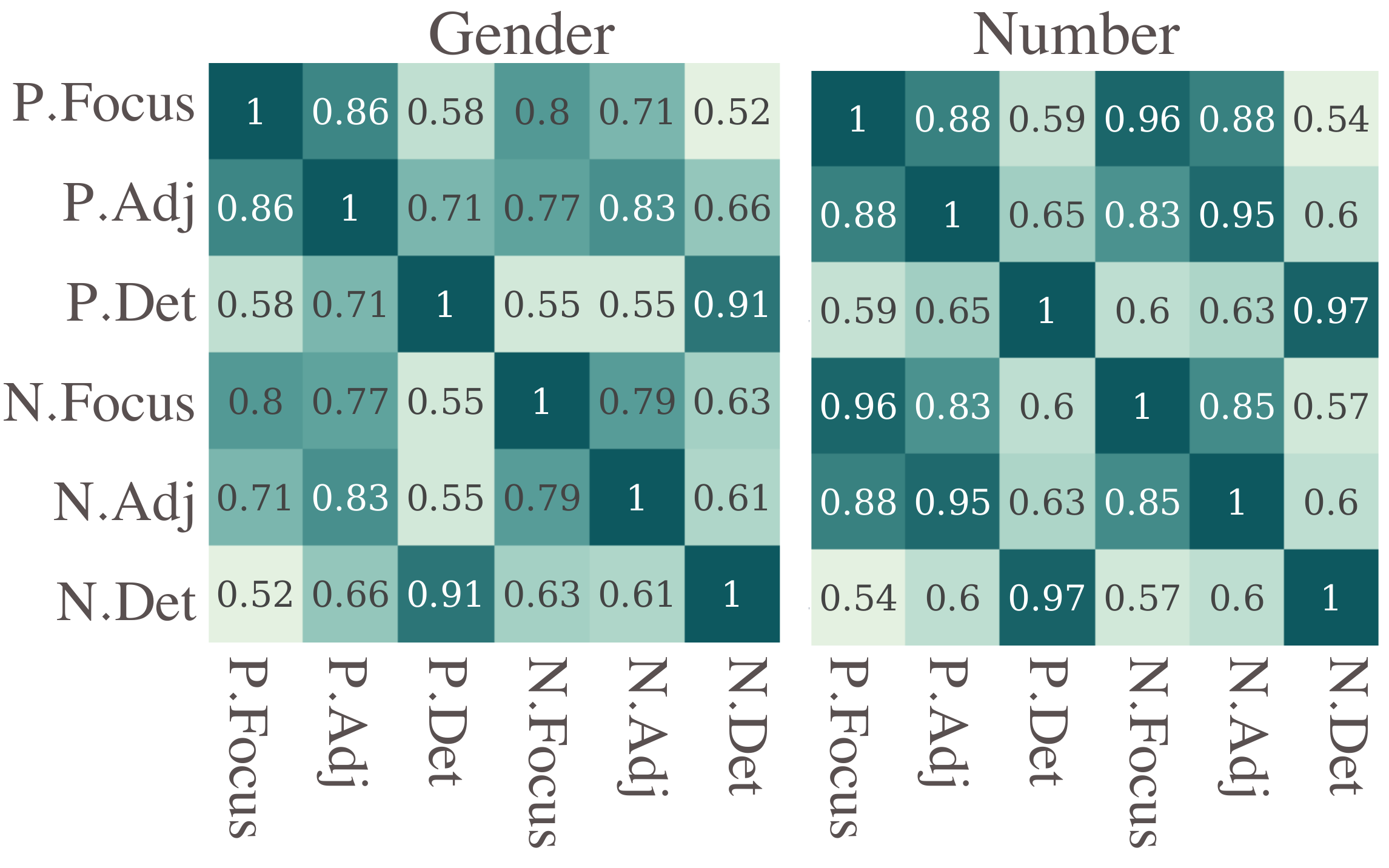}
    \caption{Cosine similarity of ATE estimators computed on focus nouns, adjectives and determiners using BERT representations.}
    \label{fig:cos-det-adj}
    \vspace{-7pt}
\end{figure}

\paragraph{Variation Across Part-of-Speech Tags.} Using the same approach, we additionally compute the ATEs 
on adjectives and determiners. 
\cref{fig:cos-det-adj} presents our na\"ive and paired ATE estimates, computed on words with different parts of speech.
These results suggest that gender and number do not affect the focus noun or its dependent words in the same way.
While the ATE on focus nouns and adjectives are strongly aligned, the cosine similarity between ATEs on focus nouns and determiners is smaller.\footnote{Relatedly, \citet{lasri-etal-2022-probing} recently showed BERT encodes information about number differently for nouns and verbs.}

\subsection{Masked Language Modeling Predictions}

We now analyze the effect of our morpho-syntactic features on masked language modeling predictions. 
Specifically, we analyze RoBERTa \citep{conneau2019unsupervised}\footnote{More specifically, we use \textsc{xlm-roberta-base}.} in these experiments, as it has better performance than BERT in masked prediction.
We thus look at how grammatical gender and number affect the probability RoBERTa assigns to each word in its output vocabulary.

\newcommand{\maskedpred}{\mathrm{MProbs}}

\begin{table}[t]
\centering
\resizebox{\columnwidth}{!}{%
\begin{tabular}{llccc@{}}
\cmidrule[\heavyrulewidth](l){2-5}
& & $\maskedpred(\bh)$ & $\maskedpred(\bhatbaseline)$ & $\maskedpred(\bhatpaired)$  \\
\cmidrule(l){2-5}
\ldelim\{{3}{2.5mm}[\parbox{3mm}{\rotatebox[origin=c]{90}{\textsc{gender}}}] 
&\textsc{Det}: $\maskedpred(\bhp)$ & $4.85 \pm 2.39$ & $1.09 \pm 1.4$ & $0.67 \pm 1.14$  \\
&\textsc{Adj}: $\maskedpred(\bhp)$ & $2.29 \pm 2$ & $1.04 \pm 1.05$ & $0.9 \pm 1.12$  \\
&\textsc{Focus}: $\maskedpred(\bhp)$ & $3.75 \pm 2.67$ & $1.74 \pm 1.11$ & $1.53 \pm 0.93$ \\ \cmidrule(l){2-5}
\ldelim\{{3}{2.5mm}[\parbox{3mm}{\rotatebox[origin=c]{90}{\textsc{number}}}] 
&\textsc{Det}: $\maskedpred(\bhp)$ & $6.93 \pm 2.52$ & $1.92 \pm 2.87$ & $2.05 \pm 2.64 $ \\
&\textsc{Adj}: $\maskedpred(\bhp)$ & $5.63 \pm 2.75$ & $2.25 \pm 2.2$ & $2.5 \pm 2.17$ \\
&\textsc{Focus}: $\maskedpred(\bhp)$ & $5.50 \pm 3.02$ & $2.25 \pm 2.14$  & $2.41 \pm 1.9$ \\
\cmidrule[\heavyrulewidth](l){2-5}
\end{tabular}
}
\caption{Mean and standard deviation of Jensen--Shannon divergence between the masked probability distributions of focus nouns, determiners and adjectives over the corpus.}
\label{tab:js}
    \vspace{-7pt}
\end{table}

We start by masking a word in our sentence: either the focus noun, a dependent determiner, or an adjective.
We then obtain this word's contextual representation $\bh$.
Second, we apply a syntactic intervention to this sentence, and, following similar steps, obtain another representation $\bhp$.
Third, we use these representations to obtain the probabilities RoBERTa assigns to the words in its vocabulary $\maskedpred(\bh)$ and $\maskedpred(\bhp)$. 
Finally, we get these same probability assignments, but using ATE to estimate the counterfactual representations:
\begin{align}
    \maskedpred(\bhatpaired),\quad & \bhatpaired = \bh \pm \paired \\
    \maskedpred(\bhatbaseline),\quad &  \bhatbaseline = \bh \pm \baseline
\end{align}

\begin{figure*}[t]
    \centering
    \includegraphics[width=\linewidth]{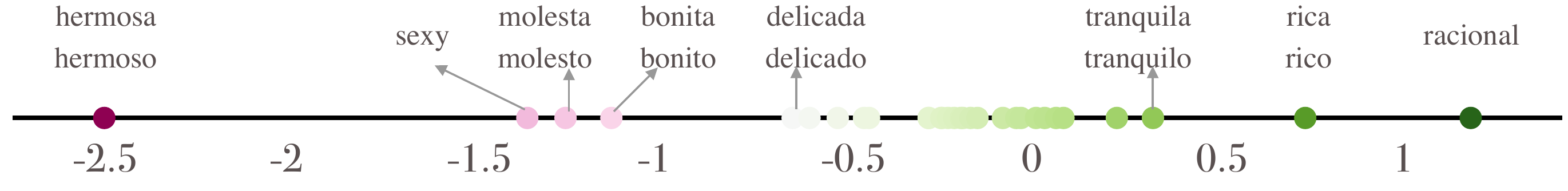} 
    \caption{$\paired^{(a)}$ values computed using \cref{eq:adj-bias} to measure causal gender bias in masked adjective prediction.}
    \label{fig:adj-bias}
    \vspace{-10pt}
\end{figure*}

We now look at how probability assignments change as a function of our interventions.
Specifically, \cref{tab:js} shows Jensen-–Shannon divergences between $\maskedpred(\cdot)$ computed on top of different representations. We can make a number of observations based on this table. First, for gender, these distributions change more when predicting determiners and focus nouns than adjectives.
We speculate that this may be because many Spanish adjectives are syncretic, i.e., they have the same inflected form for masculine and feminine, e.g., \textit{inteligente} (intelligent), or \textit{profesional} (professional).
Second, the distributions change more after an intervention on number than on gender. 
Third, when we use either of our estimators to approximate the counterfactual representation,
the divergences are greatly reduced; These results show that the ATE values do describe (at least to some extent) the change of representations in individual sentences.\looseness=-1

\subsection{Gender Bias in Adjectives}

As shown by \citet{bartl-etal-2020-unmasking} and \citet{gonen-etal-2022-analyzing}, the results of studies on gender bias in English are not completely transferable to gender-marking languages.
We analyze the causal effect of gender on specific masked adjective probabilities, predicted by the RoBERTa model.
To this end, we manually create a list of 30 adjectives (the complete list is in \cref{app:adj-list}) in both masculine and feminine forms.
We sample a sentence $\bff$ from a subset of the dataset in which the focus noun has one dependent adjective $a$, and mask this adjective.
We then define a new function $\rep(\cdot)$ to measure the ATE on adjective probabilities.
Specifically, we write:\looseness=-1
\begin{align} \label{eq:adj-bias-1}
    \rep_a(\bff) &= \ln\, p_{\theta}(a \mid \bff) \\ \nonumber
    &= \ln\,p_{\theta}(a \mid g_*, \ell_*, \bz)
\end{align}
where $a$ represents an adjective in our list (that also exists in RoBERTa's vocabulary $\calV$) and $p_\theta(a \mid \bff)$ is the probability RoBERTa assigns to that adjective.\footnote{When an adjective in the list has two forms depending on the gender (e.g., \textit{hermos\textbf{a}/hermos\textbf{o}}), we sum the probabilities for masculine and feminine forms.}
We plug this new function into our paired ATE estimator in \cref{eq:ate-paired}.
As this prediction is somewhat susceptible to noise, we replace the mean in \cref{eq:ate-paired} with the median, i.e., we compute\looseness=-1:
\begin{equation} \label{eq:adj-bias}
    \paired^{(a)} = \underset{\langle \_, \ell_*, \bz \rangle \in \mathcal{S}}{\mathrm{median}} \left[ \ln\, \frac{p_{\theta}(a \mid \masc, \ell_*, \bz)}{p_{\theta}(a \mid \fem, \ell_*, \bz)} \right]
\end{equation}
In this equation, if $\paired^{(a)} > 0$, the predicted probability that the adjective appears in a sentence where it is dependent on a masculine focus noun will be typically higher than in a sentence with a feminine focus noun. Whereas if $\paired^{(a)} < 0$ the reverse will hold. Therefore, we say $a$ is biased towards masculine gender if $\paired^{(a)} > 0$ and it is biased towards feminine gender if $\paired^{(a)} < 0$ 
As shown in \cref{fig:adj-bias}, rich (\textit{ric\textbf{a}/ric\textbf{o}}) and rational (\textit{racional}) are more biased towards masculine gender, while beautiful (\textit{hermos\textbf{a}/hermos\textbf{o}}) is biased towards feminine gender.

\begin{figure*}[t]
    \centering
    \begin{subfigure}[b]{0.32\textwidth}
        \includegraphics[width=\linewidth]{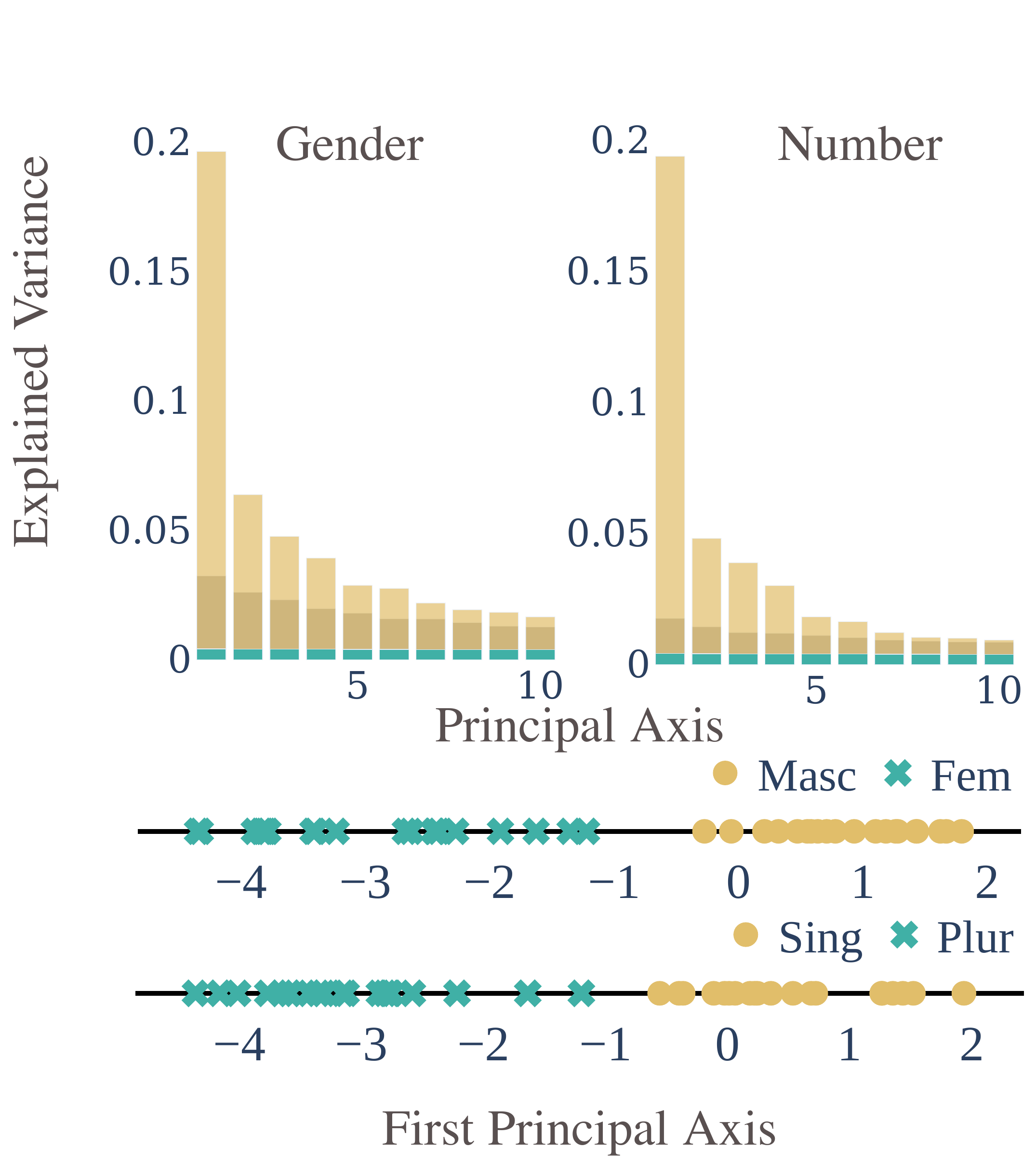} 
        \caption{BERT}
    \end{subfigure}%
    ~
    \begin{subfigure}[b]{0.32\textwidth}
        \includegraphics[width=\linewidth]{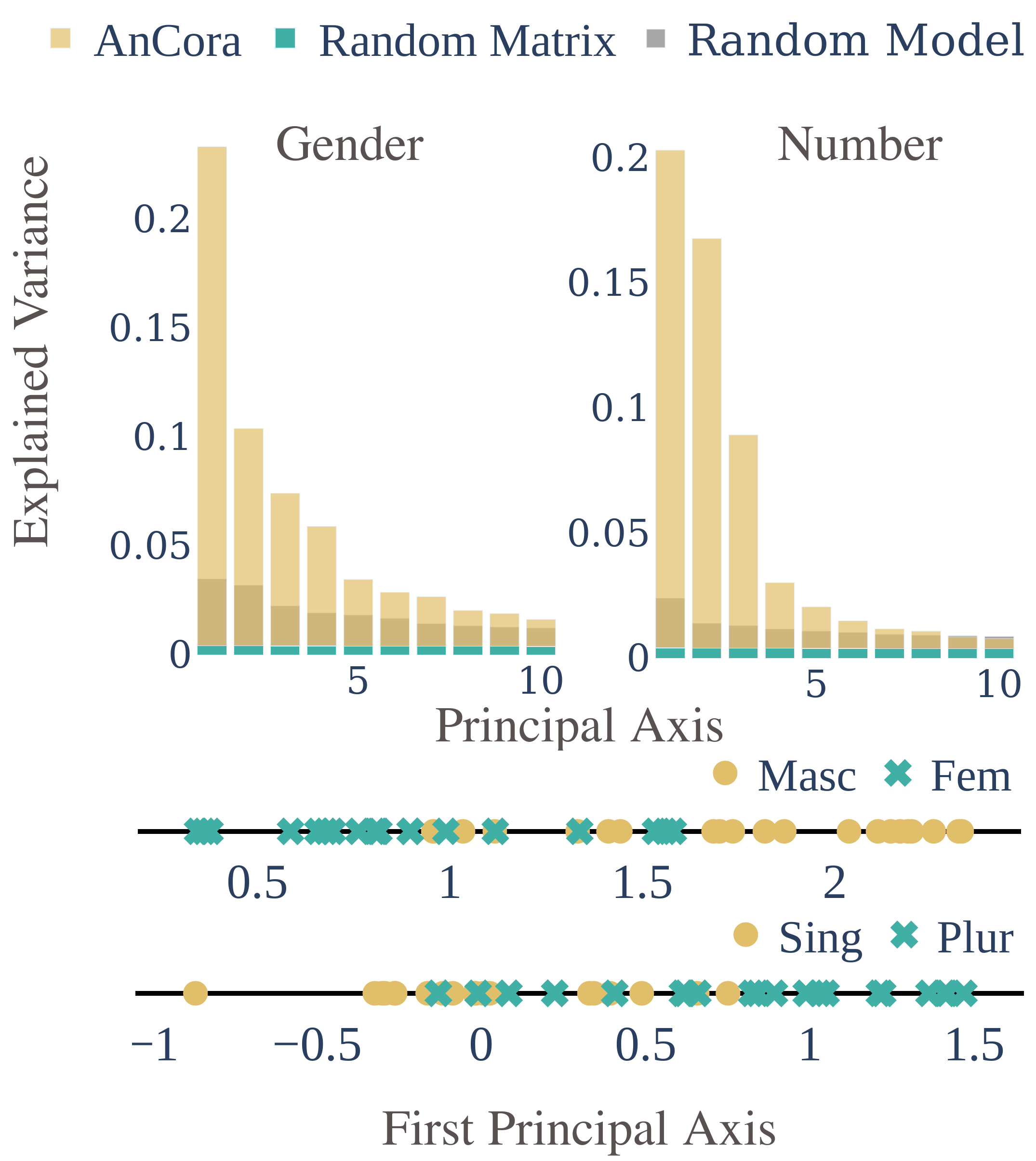} 
        \caption{RoBERTa}
    \end{subfigure}%
    ~
    \begin{subfigure}[b]{0.32\textwidth}
        \includegraphics[width=\linewidth]{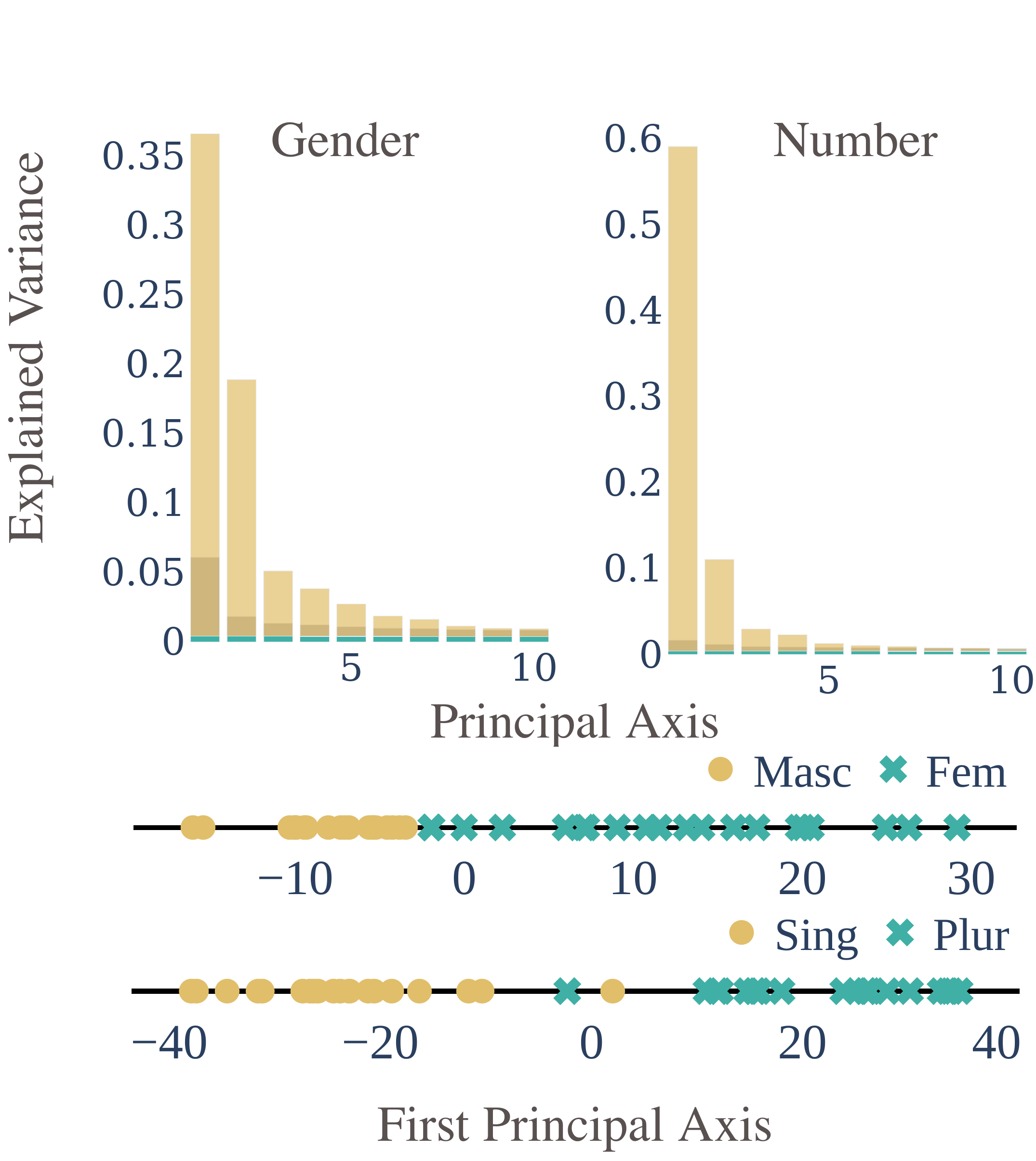} 
        \caption{GPT-2}
    \end{subfigure}
    \vspace{-6pt}
    \caption{(top) Percentage of the gender and number variance explained by the first 10 PCA components. (bottom) The projection of 20 pairs of focus noun's representations on the first principal component.}
    \vspace{-.7em}
    \label{fig:pca}
\end{figure*}

\section{Insights From Naturalistic Counterfactuals}
In the following experiments, we rely on a dataset augmented with naturalistic counterfactuals.
We first explore the geometry of the encoded morpho-syntactic features.
We then run a more classic correlational probing experiment, 
highlighting the importance of a causal framework when analyzing representations.

\subsection{Geometry of Morpho-Syntactic Features}

In this experiment, we follow \citeposs{bolukbasi2016man} methodology to isolate the subspace capturing our morpho-syntactic features' information.
First, we create a matrix with the representations of all focus nouns in our counterfactually augmented dataset. 
Second, we pair each noun's representation with its counterfactual representation (after the intervention).
Third, we center the matrix of representations by subtracting each pair's mean.
Finally, we perform principal component analysis on this new matrix.\looseness=-1

As \cref{fig:pca} shows, in BERT and RoBERTa, the first principal component explains close to 20\% of the variance caused by gender and number.
In GPT-2 \citep{radford2019language},\footnote{More specifically, we use \textsc{gpt2-small-spanish}.} more than half of the variance is captured by the first or the first two principal components.\footnote{These results are not obtained due to the randomness of a finite sample of high dimensional vectors. Neither are they due to the structure of the model. To show this, we present two random baselines: random vectors of the same size $|\mathcal{S}|$ (as green traces) and representations extracted from models with randomized weights (as gray traces) in \cref{fig:pca}.}
This result is in line with prior work  \citep[e.g.,][on Italian word embeddings]{Biasion2020GenderBI}, and suggests that these morpho-syntactic features are linearly encoded in the representations.\looseness=-1

To further explore the gender and number subspaces, we project a random sample of 20 sentences (along with their counterfactuals) onto the first principal component.
\cref{fig:pca} (bottom) shows that the three models we probe can (at least to a large extent) differentiate both morpho-syntactic features using a single dimension. 
Notably, this first principal component is strongly aligned with the estimate $\paired$; they have a cosine similarity of roughly 0.99 in all these settings.\looseness=-1

\begin{figure*}
    \centering
    \includegraphics[width=\linewidth]{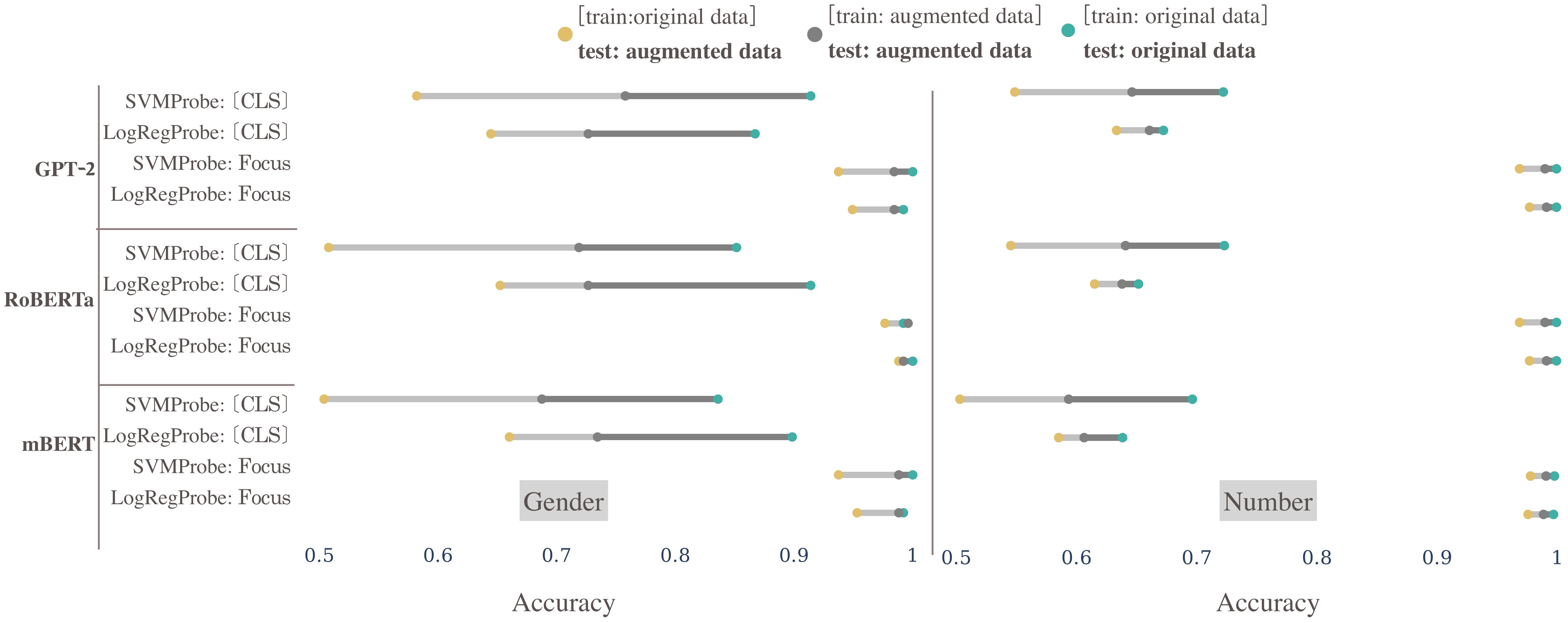}
    \caption{Accuracy scores of gender and number probes on the original and augmented datasets.}\label{fig:probe}
    \vspace{-1em}
\end{figure*}

\subsection{Analysis of Correlational Probing}
We now use a dataset augmented with naturalistic counterfactuals to empirically evaluate the entanglement of correlation and causation discussed in \cref{sec:probes}, which arises when using diagnostic probes to probe the representations.
Again, we probe three contextual representations: BERT, RoBERTa, and GPT-2.
We train logistic regressors (LogRegProbe) and support vector machines (SVMProbe) to predict either gender or number of the focus noun from its contextual representation. 
Further, we probe the representations in two positions: the focus noun and the \cls token (or a sentence's last token, for GPT-2).\footnote{BERT and RoBERTa treat \cls as a special token whose representation is supposed to aggregate information from the whole input sentence. In GPT-2, the last token in a sentence should also contain information about all its previous tokens.}\looseness=-1

Accuracy of correlational probes on the original dataset is shown in \cref{fig:probe} as green points. 
Both gender and number probes reach a near-perfect accuracy on focus nouns' representations.
Furthermore, all correlational gender probes reach a high accuracy in \cls representations, suggesting that gender can be reliably recovered from them.\looseness=-1

Next, we evaluate trained probes on counterfactually augmented test sets (shown as yellow points in \cref{fig:probe}).
We see that there is a drop in performance in all settings, and more specifically, the accuracy of probes on \cls representations drops significantly when evaluated on the counterfactual test set. This suggest that the previous results using correlational probes overestimate the extent to which gender and number can be predicted from the representations.\looseness=-1

Finally, we also \emph{train} supervised probes on a counterfactually augmented dataset in order to study whether we can achieve the levels of performance attested in the literature (shown as gray points in \cref{fig:probe}).
Since these probes are trained on a dataset augmented with counterfactuals, they are not as susceptible to spurious correlations; we thus call them the causal probes.
Although there is a considerable improvement in accuracy, there is still a large gap between correlational and causal probes' accuracies.
Together, these results imply that correlational probes are sensitive to spurious correlations in the data (such as the semantic context in which nouns appear), and do not learn to predict grammatical gender robustly.\looseness=-1

\section{Conclusion}
We propose a heuristic algorithm for syntactic intervention which, when applied to naturalistic data, allows us to create naturalistic counterfactuals. 
Although similar analyses have been run by prior work, using either templated or representational counterfactuals \citep[\textit{inter alia}]{elazar2020bert,vig2020investigating,bolukbasi2016man}, our syntactic intervention approach allows us to run these analyses on naturalistic data.
We further discuss how to use these counterfactuals in a causal setting to probe for morpho-syntax.
Experimentally, we first showed that ATE estimates are more robust to dataset differences than either our na\"ive (correlational) estimator, or template-based approaches.
Second, we showed that ATE can (at least partially) predict how representations will be affected after intervention on gender or number.
Third, we employ our ATE framework to study gender bias, finding a list of adjectives that are biased towards one or other gender.
Fourth, we find that the variation of gender and number can be captured by a few principal axes in the nouns' representations.
And, finally, we highlight the importance of causal analyses when probing: When evaluated on counterfactually augmented data, correlational probe results drop significantly.\looseness=-1

\section*{Ethical Concerns}
Pretrained models often encode gender bias.
The adjective bias experiments in this work can provide further insights into the extent to which these biases are encoded in multilingual pretrained models. 
As our paper focuses on (grammatical) gender as a morpho-syntactic feature, it focuses on a binary notion of gender, which is not representative of the full spectrum of human gender expression.
Most of the analysis in this paper focuses on measuring grammatical gender, not gender bias. 
We thus advise caution when interpreting the findings from this work.
Nonetheless, we hope the causal structure formalized here, together with our analyses, can be of use to bias mitigation techniques in future \citep[e.g.,][]{liang-etal-2020-towards}.\looseness=-1

\section*{Acknowledgments}
We would like to thank Shauli Ravfogel for feedback on a preliminary draft and Dami{\'a}n Blasi for analyzing the errors made by our naturalistic counterfactual algorithm.
We would also like to thank Miguel Ballesteros, who served as a marvelous action editor, and the the anonymous reviewers for their insightful feedback during the review process. 
Afra Amini is supported by ETH AI Center doctoral fellowship.
Ryan Cotterell acknowledges support from the SNSF through the ``The Forgotten Role of Inductive Bias in Interpretability'' project.

\appendix
\section{List of Adjectives}\label{app:adj-list}
We use 30 different Spanish adjectives in our experiments: \textit{hermoso/hermosa (beautiful),
sexy (sexy),
molest/molesta (upset),
bonito/bonita (pretty),
delicado/delicada (delicate),
rápido/rápida (fast),
joven (young),
inteligente (intelligent),
divertido/divertida (funny),
fuerte (strong),
duro/dura (hard),
alegre (cheerful),
protegido/protegida (protected),
excelente (excellent),
nuevo/nueva (new),
serio/seria (serious),
sensible (sensitive),
profesional (professional),
emocional (emotional),
independiente (independent),
fantástico/fantástica (fantastic),
brutal (brutal),
malo/mala (bad),
bueno/buena (good),
horrible (horrible),
triste (sad),
amable (nice),
tranquilo/tranquila (quiet),
rico/rica (rich),
racional (rational)}.\looseness=-1
\onecolumn
\section{Algorithm for Heuristic Intervention} \label{app:alg}
\begin{algorithmic}[1]
\footnotesize
\Procedure{ReinflectTree}{\node, \parent, \istate}
 \State \isFocusNoun $\; \gets \; $ \false
 \If {\istate $==$ \normal and \node$\;$ is a valid noun}
 \State \textsc{ReinflectNoun}(\node) \rightcomment{Change the noun and set the morpho-syntactic feature to the desired value}
 \State \isFocusNoun $\; \gets \; $ \true
 \If {\node $\;$ is subject}
 \State \textsc{ReinflectVerb}(\parent) \rightcomment{Change verb}
 \EndIf
 \EndIf
 \If {\istate $==$ \dir} \rightcomment{Current node is a direct dependent of a focus noun}
    \If {\node $\;$ is a determiner}
        \State \textsc{ReinflectDet}(\node) \rightcomment{Change determiner}
    \EndIf
    \If {\node $\;$ is an adjective modifier}
        \State \textsc{ReinflectAdj}(\node) \rightcomment{Change adjective}
    \EndIf
    \If {\node $\;$ is a nominal subject}
        \State \textsc{ReinflectNoun}(\node) \rightcomment{Change noun}
        \State \isNSubj $\; \gets \; $ \true
    \EndIf
    \If {\node $\;$ is a copula}
        \State \textsc{ReinflectCop}(\node) \rightcomment{Change copula}
    \EndIf
 \EndIf
 
 \If {\istate $==$ \indir and \node $\;$ is an adjective modifier and \parent $\;$ is an adjective modifier} \rightcomment{Current node is a descendant of a focus noun}
 \State \textsc{ReinflectAdj}(\node)
 \EndIf
 \For {$\mathtt{child}$ $\in$ children(\node)}
 \If {\isFocusNoun $\;$ or \isNSubj}
 \State \textsc{ReinflectTree}($\mathtt{child}$, \node, \dir)
 \ElsIf {\istate $==$ \dir or \istate $==$ \indir}
 \State \textsc{ReinflectTree}($\mathtt{child}$, \node, \indir)
 \Else 
 \State \textsc{ReinflectTree}($\mathtt{child}$, \node, \normal)
 \EndIf
 \EndFor
\EndProcedure
\end{algorithmic}
\section{Theory}
\begin{prop}\label{prop1} {\normalfont In this proposition we show that the average treatment effect is equivalent to the difference of two expectations with no do-operator:}
{\normalfont
    \begin{align} \label{eq:prop1}
     \E_{\bF} \Big[ \rep(\bF) \mid &\,\mathrm{do}\left(G_* = \masc\right) \Big] -\E_{\bF} \Big[ \rep(\bF) \mid \mathrm{do}\left(G_* = \fem\right)\Big] \\ 
     &= \E_{L_*, \bZ} \left[\E_{\bF} \Big[ \rep(\bF) \mid G_* = \masc, L_*, \bZ \Big]\right] -\E_{L_*, \bZ} \left[\E_{\bF} \Big[ \rep(\bF) \mid G_* = \fem, L_*, \bZ \Big] \right] \nonumber
\end{align}
}
\end{prop}
\begin{proof}
First, we note the existence of two backdoor paths in our model \cref{fig:causal-abstract}: $M_* \leftarrow U \rightarrow \bZ \rightarrow \bF \rightarrow \bR$ and $M_* \leftarrow U \rightarrow L_* \rightarrow \bF \rightarrow \bR$. We can easily check that $\bZ$ blocks the first and $L_*$ blocks the second path, and neither $\bZ$ nor $L_*$ are descendants of $M_*$. Therefore $\{L_*, \boldsymbol{Z}\}$ satisfies the back-door criterion.
To make the proof simpler, we show that the first term of the left-hand side of \cref{eq:prop1} equals the first term in the right-hand side of \cref{eq:prop1} and then we get the full  result by symmetry.
We proceed as follows:
\begin{align}
    \E_{\bF} \Big[ \rep(\bF) \mid & \mathrm{do}\left(G_* = \masc\right) \Big] \\
    &= \sum_{\ell_* \in \mathcal{L}} \sum_{\bz \in \mathcal{Z}}\E_{\bF} \Big[ \rep(\bF) \mid \doo(G_* = \masc), \ell_*, \bz)\Big] p(\ell_*, \bz) \quad &\text{{\color{gray} (marginalize $\ell_*$ and $\bz$)}} \nonumber \\
    &= \sum_{\ell_* \in \mathcal{L}} \sum_{\bz \in \mathcal{Z}}\E_{\bF} \Big[ \rep(\bF) \mid G_* = \masc, \ell_*, \bz)\Big] p(\ell_*, \bz) \quad &\text{{\color{gray} (backdoor criterion)}} \nonumber \\
    &=   \E_{L_*, \bZ} \left[\E_{\bF} \Big[ \rep(\bF) \mid G_* = \masc, L_*, \bZ \Big]\right] &\text{{\color{gray} (rewrite as an expectation)}} \nonumber
\end{align}
\end{proof}

\clearpage
\newpage
\twocolumn
\bibliographystyle{acl_natbib.bst}
\bibliography{anthology.bib}

\end{document}